\documentclass{article}
\usepackage{ijcai17}

\usepackage{times}
\usepackage{helvet}
\usepackage{courier}
\usepackage[linesnumbered,ruled,vlined]{algorithm2e}
\usepackage{graphicx}
\usepackage{amsmath}
\usepackage{mathtools}
\usepackage{amsfonts}
\usepackage{amssymb}
\usepackage{multirow}
\usepackage{varwidth}
\usepackage{hhline}
\usepackage[table,xcdraw]{xcolor}
\usepackage{subfig}
\usepackage{amsthm}
\usepackage{enumitem}
\usepackage{url}
\DeclareMathOperator*{\argmax}{argmax}
\begin{document}

\setlength{\floatsep}{-0.2pt}
\setlength{\itemsep}{0pt}
\setlength{\parsep}{-0.1pt}
\setlength{\textfloatsep}{0cm}

\title{Quantifying Aspect Bias in Ordinal Ratings using a Bayesian Approach}

\author{Lahari Poddar \hspace{1cm} Wynne Hsu \hspace{1cm} Mong Li Lee \\
        School of Computing  \\ National University of Singapore \\ \{lahari, whsu, leeml\}@comp.nus.edu.sg}

\maketitle

\newtheorem{theorem}{Theorem}[section]
\newtheorem{lemma}[theorem]{Lemma}

\begin{abstract}
User opinions expressed in the form of ratings can influence an individual's view of an item.
However, the true quality of an item is often obfuscated by user biases, and it is not obvious from the observed ratings
the  importance different users place on different aspects of an item.
We propose a probabilistic modeling of the observed aspect ratings to infer (i) each user's aspect bias and (ii) latent intrinsic quality of an item. We model multi-aspect ratings as ordered discrete data and encode the dependency between different aspects by using a latent Gaussian structure. We handle the Gaussian-Categorical non-conjugacy using a stick-breaking formulation coupled with P\'{o}lya-Gamma auxiliary variable augmentation for a simple, fully Bayesian inference. 
On two real world datasets, we demonstrate
the predictive ability of our model and its effectiveness in learning explainable user biases to provide insights towards a more reliable product quality estimation.

\end{abstract}

\section{Introduction}
With easy availability of information on the web, user ratings have become increasingly important in molding people's perception of an item.
However, an item typically has many aspects and not all aspects are equally important to all users. 
To some user, the \textit{cleanliness} of a hotel is most important and he/she tends to rate this aspect stringently, but is lenient when rating \textit{food} or \textit{amenities}. Other users may have a different set of preferences and their aspect ratings for the same item could be vastly different. Hence, it is difficult to interpret conflicting ratings without knowing the underlying user biases. For an item with only few ratings this is aggravated, since even its average ratings are highly susceptible to the users' biases. 

\begin{figure}[thbp]
	\centering
\includegraphics[width=0.75\linewidth]{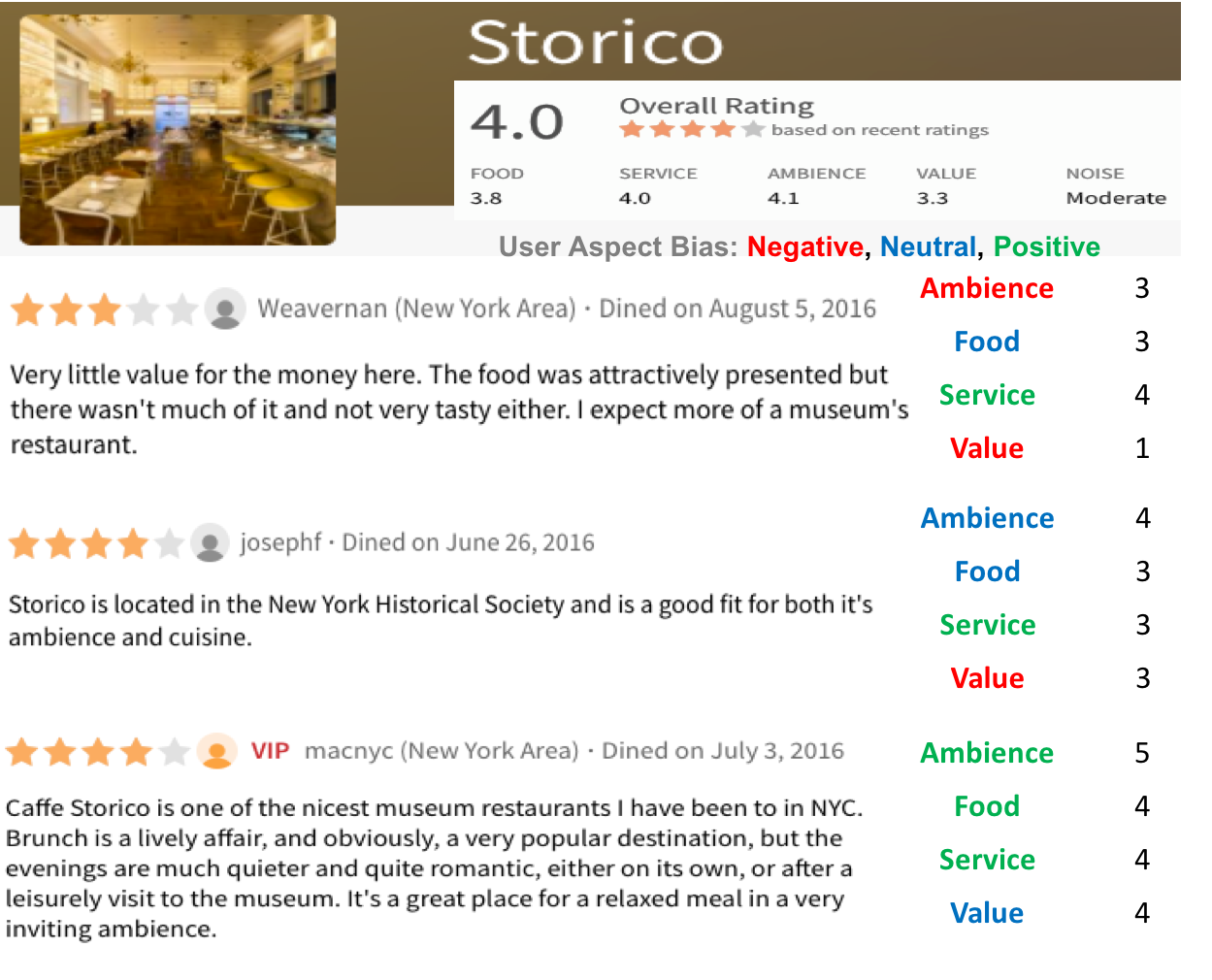}\label{sample}
\vspace{-0.1in}	
	\caption{\small A sample restaurant's ratings with color coded user aspect bias. This is a real output produced by the proposed model.  }
	\vspace*{0.1in}
\label{demo}
\end{figure}

To enable proper interpretation of ratings, 
we propose a unified probabilistic model for quantifying the underlying user biases for different aspects that lead to the  observed ratings. 
We model the correlation between aspects by allowing a covariance structure among them. This is realistic since a user's bias, and in turn his rating, of one aspect may be correlated with another aspect.

We detect the underlying aspect preferences of individual users that are consistent across their ratings on different items.
We can learn the aspect bias of users even with few ratings, by introducing latent user groups, based on the similarity of users' rating behavior on various aspects. 
 For example, one user group might generally give low ratings for \textit{ambience} while another user group  gives high ratings for \textit{food}.

 Figure \ref{demo} shows an example application of the model where the learned user aspect bias is displayed beside the ratings.
People with a negative bias tend to be more critical about the aspect and generally underrate the aspect than other users, whereas people with a positive bias for an aspect tend to overrate it. Knowing the aspect biases of individuals, users can better interpret their ratings.
Furthermore this is beneficial for service providers to focus on improving the aspects of an item that  consumers \emph{truly} care about. 

While existing works assume ratings to be continuous, in reality most observed ratings in e-commerce websites are ordinal in nature. 
Our model  incorporates the ordinal nature of observed ratings through proper statistical formulation.
However, modeling the ordinal nature of observed ratings as well the correlation between aspects introduce non-conjugacy into our model, making Bayesian inference very challenging.  
 
To eliminate the non-conjugacy of Gaussian prior-Categorical likelihood, we utilize stick-breaking formulation with  P\'{o}lya-Gamma auxiliary variable augmentation.
The construction proposed in the paper is efficient and generic. It will help developing inference mechanisms for various applications that need to model ordinal data  in terms of continuous latent variables with a correlation structure.

Experiments on two real world datasets from TripAdvisor and OpenTable
demonstrate that the proposed model provides new insights in users' rating patterns, and  outperforms state-of-the-art methods for aspect rating prediction. 

To the best of our knowledge, this is the first work to model ordinal aspect ratings parameterized by latent multivariate continuous responses, with a simple, scalable and fully Bayesian inference.

\section{Ordinal Aspect Bias Model}

In this section, we describe the design of our Ordinal Aspect Bias model and present a Bayesian approach for inference.

Suppose we have $J$ users and $I$ items. 
Let  $R$ be the set of observed ratings where  $\mathbf{r_{ij}}$ is an $A$ dimensional vector denoting the rating of user $j$ for item $i$ on each of its aspects. Each $\mathbf{r_{ij}}$ is a discrete value between $1$ and $K$  corresponding to a $K$-level scale (\textit{poor} to \textit{excellent}).  We assume that $\mathbf{r_{ij}}$ arises from a latent multivariate continuous response $\mathbf{v_{ij}}$ which is dependent on (i) the intrinsic quality of the item on the aspect and (ii) the bias of the user for the aspect.

The intrinsic quality of an item $\mathbf{z_i}$ is an $A$ dimensional vector, drawn from a multivariate normal distribution, with mean $\boldsymbol{\mu}$ and covariance matrix $\boldsymbol{\Sigma}$. We use multivariate normal distribution to account for the correlation among the subsets of aspects of an item. For example, it is highly unlikely for a hotel to have excellent \textit{room} quality but very poor \textit{cleanliness}, but it is possible to have a  good \textit{location} and average \textit{food} choices. Such correlations among subset of aspects are captured by the covariance matrix. The parameters $(\boldsymbol{\mu, \Sigma})$ are given a conjugate normal-inverse Wishart (NIW) prior. 

The preference of a user for an aspect is captured by a bias vector $\mathbf{m_g}$ of dimension $A$. If a user places great importance on a particular aspect (e.g. \textit{cleanliness}), this will be reflected in his ratings across all hotels.
In other words, his rating on the \textit{cleanliness} aspect will tend to be lower than the majority's rating for \textit{cleanliness} on the same hotel. We cluster users with similar preferences into different user groups and associate a bias vector $\mathbf{m_g}$ with each group. The membership of a user $j$ in a user group is denoted as $s_j$  where $s_j$ is drawn from a categorical distribution $\theta$ with a Dirichlet prior parameter $\alpha$. 

Given the intrinsic quality $\mathbf{z_i}$ and  bias $\mathbf{m_g}$,
the latent response  $\mathbf{v_{ij}}$ is drawn from a multivariate Gaussian distribution with $\mathbf{z_{i}} + \mathbf{m_{s_j}}$ as mean and a hyper-parameter $\mathbf{B}$ as covariance. This is intuitive as a user's response depends on the item's intrinsic quality for an aspect as well as his own bias. 

With the latent response $\mathbf{v_{ij}}$, we sample the observed rating vector $\mathbf{r_{ij}}$. 
Note that since the observed ratings are ordered and discrete, they should  be drawn from a categorical distribution. 
However, the latent response $\mathbf{v_{ij}}$ is given a multivariate Gaussian prior. 
In order to have a fully Bayesian inference, we need to \textit{transform} this categorical distribution to a Gaussian form to exploit conjugacy. This is the central technical challenge for our proposed model. 

We develop a stick-breaking mechanism with logit function to map the categorical likelihood to a binomial form.
Thereafter, leveraging the recently developed P\'{o}lya-Gamma auxiliary variable augmentation scheme \cite{polson2013bayesian}, the binomial likelihood is transformed to Gaussian, thus establishing conjugacy and enabling us to achieve an effective posterior inference. 
The generative process of the model is as follows:

\begin{enumerate}
	\small
	\itemsep0em 
	\item Draw a multinomial group distribution $\theta$ from Dirichlet $(\alpha)$.
	\item For each group $g \in {1, \cdots , G}$ draw a bias offset $\mathbf{m_g}$ from $N_A(0,\Lambda)$
	\item For each user $j \in {1, \cdots, J} $, sample a group $s_j$ from Cat ($\theta$)
	\item For each item $i \in {1, \cdots, I} $, sample an intrinsic rating $\bf{z_i}$ from $N_A(\boldsymbol{\mu, \Sigma})$
	\item For each rating $\mathbf{r_{ij}} \in R$
	\begin{enumerate}
		\item draw latent continuous rating $\bf{v_{ij}}$ from $N_A(\mathbf{z_i}+\mathbf{m_{s_j}} , \mathbf{B})$
		\item draw observed ordinal rating $\bf{r_{ij}}$ from
		Cat $(SB(\bf{v_{ij}},\bf{c}))$
	\end{enumerate}
\end{enumerate}

\noindent where $SB(\bf{v_{ij},c})$ refers to the stick-breaking parametrization of the continuous response $\bf{v_{ij}}$  using cut-points $\bf{c}$. Figure \ref{fig:IB} shows the proposed graphical model using plate notation. 
\begin{figure}[tbp]
	\centering
	\includegraphics[width=0.65\linewidth]{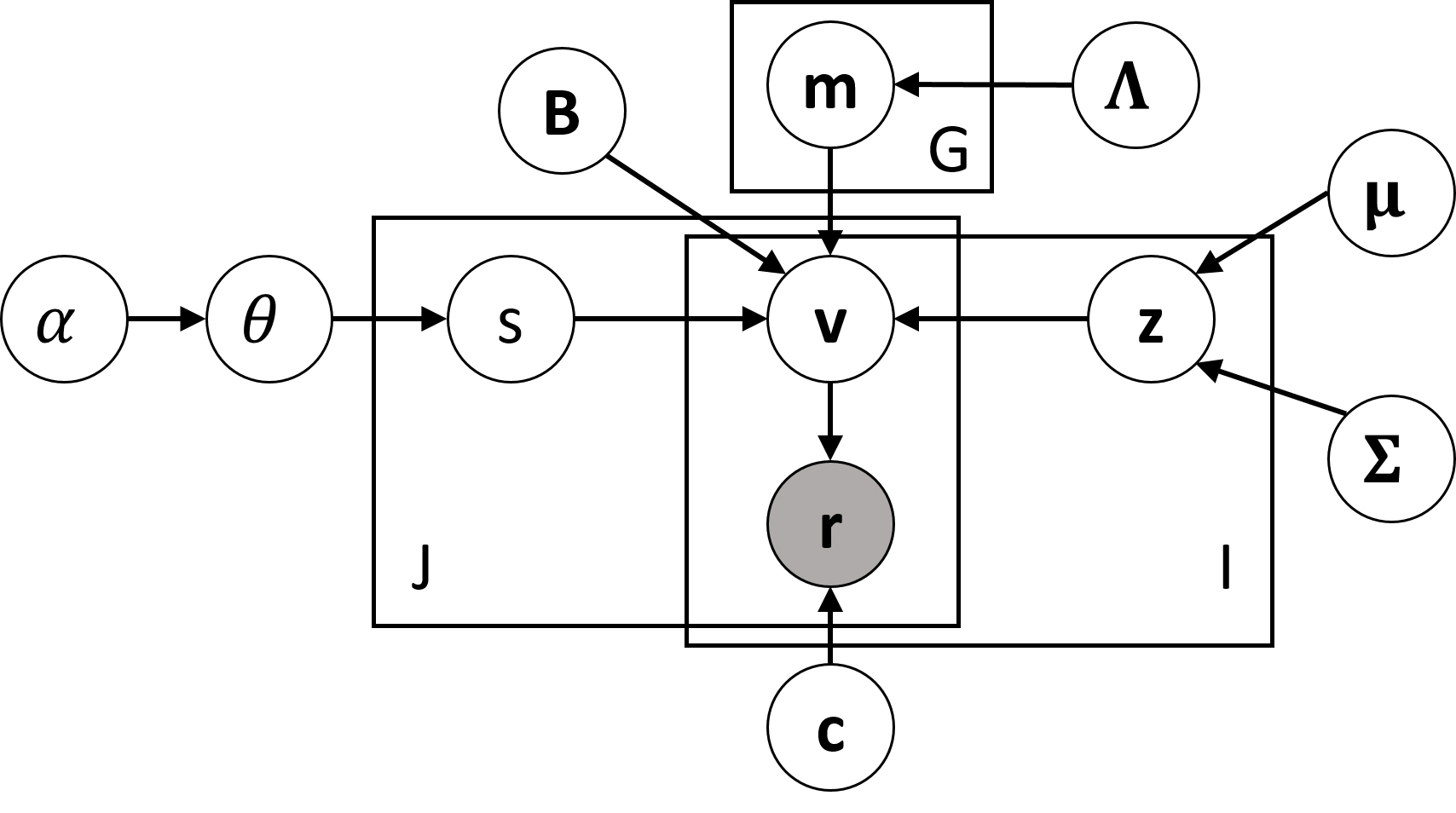}
    \vspace{-0.1in}	
	\caption{\small Ordinal Aspect Bias Model}\label{fig:IB}
	\vspace*{0.1in}
\end{figure}

\subsection{Stick-Breaking Likelihood}

We first discuss how to map the categorical likelihood of $\mathbf{v_{ij}}$, denoted as $Lik(\mathbf{v_{ij}})$, to a binomial form.

Let $r_{ija}$ denote the observed ordinal rating of item $i$, by user $j$ on aspect $a$,  and is drawn from a categorical distribution over $K$ categories. Since the categories are ordered, we utilize a stick-breaking parameterization for the probabilities $P(r_{ija} = k) \text{ where } k \in \{1, \cdots ,K\}$.  Suppose we have a unit length stick where the continuum of points on this stick represents the probability of an event occurring. If we break this stick at some random point $p$, then we have a probability mass function over two outcomes (with probabilities $p$ and $1-p$). By breaking the stick multiple times, we  obtain a probability mass function over multiple categories. 

Let $\mathbf{c}=\{c_1, \cdots, c_{K-1}\}$  be a cut-point vector where  $c_1 < c_2 < \cdots < c_{K-1}$ represent the boundaries between the ordered categories. 
The probability of each ordinal rating $r_{ija}$ being assigned the categorical value $k$, is parametrized using a function of the covariate $\eta_{ija}^k$ = $c_{k} - v_{ija}$.
Then the probability of observing the vector of ratings $\mathbf{r_{ij}}$ is a product of probabilities of observing each of the aspect ratings $r_{ija}$ given the values of $\boldsymbol{\eta_{ija}}$. 
Hence the likelihood of $\mathbf{v_{ij}}$ is:
\vspace*{-0.1in}
\begin{equation}
\small
Lik(\mathbf{v_{ij}}) = P(\mathbf{r_{ij}} | \mathbf{v_{ij} , c}) = P(\mathbf{r_{ij}} | \boldsymbol{\eta_{ij}}) = \prod_{a=1}^{A} P(r_{ija}|\boldsymbol{\eta_{ija}})
\label{Lik}
\end{equation}

To squash $\boldsymbol{\eta_{ija}}$ within [0,1] we use a sigmoid function on it denoted by $f(x) = \frac{e^x}{1+e^x}$.
Sigmoid function  enables us to use
 P\'{o}lya-Gamma augmentation scheme  \cite{polson2013bayesian} to handle the non-conjugacy subsequently. 
 For identifiability, we set $f(\eta_{ija}^K) = 1$. 
The stick-breaking likelihood can be written as: 
\begin{equation}
\small
P(r_{ija} = k) = \prod_{k' < k} (1 - f(\eta_{ija}^{k'})) f(\eta_{ija}^k)
\label{p_k}
\end{equation}
 \vspace*{-0.1in}

\noindent By encoding  $r_{ija}$ with a 1-of-$K$ vector $\mathbf{x_{ija}}$ where
\begin{equation}
x_{ija}^k = 
\begin{cases}
1 & \text{if  } r_{ija} = k \\
0 & \text{otherwise}
\end{cases}
\end{equation}

\noindent we now rewrite the likelihood of $\mathbf{v_{ij}}$  in binomial terms:
\vspace*{-0.1in}
\begin{equation}
\small
P(r_{ija}|\boldsymbol{\eta_{ija}}) =
P(\mathbf{x_{ija}} | \boldsymbol{\eta_{ija}}) = \prod_{k=1}^{K-1} Binom(x_{ija}^k | N_{ija}^k , f(\eta_{ija}^k))
\label{sb_binom}
\end{equation}
\vspace*{-0.2in}
where $$N_{ija}^k = 1 - \sum_{k'<k} x_{ija}^{k'}$$

\subsection{P\'{o}lya-Gamma Variable Augmentation}

Next, we explain how to transform the binomial likelihood to a Gaussian form via P\'{o}lya-Gamma ($PG$) auxiliary variable augmentation scheme. The integral identity at the heart of the $PG$ augmentation is:
\begin{equation}
\small
\frac{(e^{\psi})^a}{(1+e^{\psi})^b} = 2^{-b} e^{\kappa \psi} \int_{0}^{\infty} e^{- \omega \psi^2/2} p(\omega) d\omega
\label{PGintegral}
\end{equation}
where $\kappa = a -b/2$, $b>0$ and $\omega \sim PG(b,0)$. 

By expanding the binomial likelihood in Eqn. \ref{sb_binom}, we get
{\small
\begin{align}
\nonumber P(\mathbf{x_{ija}} | \boldsymbol{\eta_{ija}}) &= \prod_{k=1}^{K-1} \binom{N_{ija}^k}{x_{ija}^k} (f(\eta_{ija}^k))^{x_{ija}^k} (1 - f(\eta_{ija}^k))^{N_{ija}^k - x_{ija}^k} \\
&= \prod_{k=1}^{K-1} \binom{N_{ija}^k}{x_{ija}^k} \frac{(e^{\eta_{ija}^k})^{x_{ija}^k}}{(1 + e^{\eta_{ija}^k})^{N_{ija}^k}}\\\nonumber
\end{align}
}%

Using the integral identity of PG augmentation, we can now rewrite the categorical likelihood of $\mathbf{v_{ij}}$ as:
\vspace*{-0.1in}
{\small
\begin{align}
\label{pg_lik}
Lik(\mathbf{v_{ij}}) &= \prod_{a=1}^{A} P(\mathbf{x_{ija}}|\boldsymbol{\eta_{ija}})\\
&\propto \prod_{a=1}^{A} \prod_{k=1}^{K-1}  e^{\kappa_{ija}^k \eta_{ija}^k} \int_{0}^{\infty} e^{- \omega_{ija}^k ({\eta_{ija}^k})^2/2} p(\omega_{ija}^k) d\omega_{ija}^k \nonumber
\end{align}
}%
\noindent where $\kappa_{ija}^k = x_{ija}^k - N_{ija}^k/2$, $\psi_{ija}^k = \eta_{ija}^k$ and $p(\omega_{ija}^k)$ is $PG(N_{ija}^k/2,0)$ independent of $\psi_{ija}^k$.

By property of PG distribution \cite{polson2013bayesian},  we can draw the 
 auxiliary variable 
$\omega_{ija}^k$  from  $PG(N_{ija}^k, \eta_{ija}^k)$.
Conditioning on $\boldsymbol{\omega_{ij}}$, $Lik(\mathbf{v_{ij}})$ can be  transformed to a Gaussian form:

\vspace*{-0.05in}
{\small
\begin{align}
\label{gauss_lik}
Lik(&\mathbf{v_{ij}}) \propto \prod_{k=1}^{K-1} \prod_{a=1}^{A} e^{\kappa_{ija}^k \eta_{ija}^k} e^{- \omega_{ija}^k ({\eta_{ija}^k})^2/2} \\\nonumber
& \propto \prod_{k=1}^{K-1} \prod_{a=1}^{A} exp\{\kappa_{ija}^k (c_{k} - v_{ija}) - \omega_{ija}^k (c_{k} - v_{ija})^2 /2 \} \\ \nonumber
& \propto \prod_{k=1}^{K-1} \prod_{a=1}^{A} exp\{ - \omega_{ija}^k ((c_{k} - v_{ija}) - \frac{\kappa_{ija}^k}{\omega_{ija}^k})^2\} \\\nonumber
& 
\begin{aligned}
\propto \prod_{k=1}^{K-1} exp\{ -\frac{1}{2} (\frac{\boldsymbol{\kappa_{ij}^k}}{\boldsymbol{\omega_{ij}^k}} - (\mathbf{c_{k}} -  \mathbf{v_{ij})})^T   \boldsymbol{\Omega_{ij}^k} (\frac{\boldsymbol{\kappa_{ij}^k}}{\boldsymbol{\omega_{ij}^k}} -  (\mathbf{c_{k}} - \mathbf{v_{ij})})
\end{aligned}
\end{align}
}%
where $\boldsymbol{\kappa_{ij}^k}, \boldsymbol{\omega_{ij}^k}$ are vectors of dimension $A$,  $\boldsymbol{\Omega_{ij}^k}$ is a diagonal matrix of $(\omega_{ij1}^k , \omega_{ij2}^k, \cdots , \omega_{ijA}^k)$. 

Here,  we assume the values in the $A$-dimensional cut-point vector $\bf{c_k}$ are all equal to $c_k$. In practice, if we need different cut-points for different aspects, $\bf{c_k}$ can be set accordingly. 

\subsection{Bayesian Inference} 
Finally, we  describe the sampling of user groups $\mathbf{s}$, bias offset of user groups $\mathbf{m}$, intrinsic ratings $\mathbf{z}$,  cut-points $\mathbf{c}$ and latent continuous ratings $\mathbf{v}$ using fully Bayesian MCMC inference.
We  factor the joint probability of these  variables as:

{\small
\vspace*{-0.05in}
\begin{equation*}
P(\mathbf{r,v,m,z,s,c}) = P(\mathbf{r|v,c})P(\mathbf{v|m,z,s})P(\mathbf{c})P(\mathbf{z})P(\mathbf{s})P(\mathbf{m})
\vspace*{0.05in}
\end{equation*}
}%

\noindent\textbf{Sampling Bias Offset of User Groups.}
For each user group $g$, we sample its bias offset $\bf{m_g}$ from the Gaussian posterior:

{\small
\begin{equation*}
\small
P(\mathbf{m_g} | \boldsymbol{\Lambda},\mathbf{v},\mathbf{z}) \propto P(\mathbf{m_g} | \boldsymbol{\Lambda}) \prod_{j \in J[g]} \prod_{i \in I[j]} P(\mathbf{v_{ij}} | \mathbf{m_g} , \mathbf{z_i}, \mathbf{B})
\end{equation*}
}%

where $J[g]$ is the set of users belonging to group $g$ and $I[j]$ is the subset of items rated by user $j$.

Since the prior is a multivariate Gaussian $N_A(0,\Lambda)$ and the observations $\mathbf{v_{ij}}$ are also drawn from a multivariate Gaussian $N_A(\mathbf{z_i}+\mathbf{m_g}, \mathbf{B})$, the posterior of $\mathbf{m_g}$ is given by a Gaussian $N_A(\mathbf{\hat{m_g}}, \boldsymbol{\hat{\Lambda_g}})$ with

{\small
\vspace*{-0.05in}
\begin{eqnarray*}
\mathbf{\hat{m_g}} &=& \boldsymbol{\hat{\Lambda_g}} (\mathbf{B}^{-1} \sum_{j \in J[g]} \sum_{i \in I[j]} (\mathbf{v_{ij}} - \mathbf{z_i}))\\ 
\boldsymbol{\hat{\Lambda_g}} &=& (n_g \mathbf{B}^{-1} + \boldsymbol{\Lambda})^{-1}
\end{eqnarray*}
}%

where $n_g$ is the total number of ratings observed for users belonging to group $g$.\\

\noindent\textbf{Sampling User Groups.}
We  integrate out the group distribution $\theta$ by exploiting Dirichlet-Multinomial conjugacy, and sample the group of each user $j$ as:
{\small
\begin{equation*}
\small
P(s_j| \alpha, \mathbf{m}, \mathbf{v}) \propto  P (s_j | \alpha) \prod_{i \in I[j]} P(\mathbf{v_{ij}} | \mathbf{m_{s_j}} , \mathbf{z_i} , \mathbf{B})
\end{equation*}
}%
where $I[j]$ are the subset of items rated by user $j$, 
the prior
$P (s_j | \alpha)$ is given by the Dirichlet distribution. The likelihood is the multinomial distribution given by the probability of observing all the ratings of the  user $j$ given bias $m_{s_j}$. \\

\noindent\textbf{Sampling Intrinsic Ratings.}
Similar to the bias offsets of user groups, we sample intrinsic rating $\bf{z_i}$  of each item $i$ from a Gaussian distribution $N_A(\boldsymbol{\hat{\mu_i}} , \boldsymbol{\hat{\Sigma_i}})$ where
{\small
\begin{align*}
\boldsymbol{\hat{\mu_i}} &= \boldsymbol{\hat{\Sigma_i}} (\mathbf{B}^{-1} \sum_{j \in J[i]} (\mathbf{v_{ij}} - \mathbf{m_{s_j}}) 
+ \boldsymbol{\Sigma}^{-1} \boldsymbol{\mu}) \\
\boldsymbol{\hat{\Sigma_i}} &= (n_i \mathbf{B}^{-1} + \boldsymbol{\Sigma})^{-1}  \nonumber
\end{align*}
}
where $n_i$ is the total number of ratings observed for item $i$ and $J[i]$ is the subset of users who have rated item $i$.
The prior parameters  $\boldsymbol{\mu,\Sigma}$ of the intrinsic ratings are given a conjugate Normal-Inverse Wishart (NIW) prior and sampled.\\

\noindent\textbf{Sampling Latent Continuous Ratings.}
The latent continuous ratings, $\mathbf{v_{ij}}$ have a Gaussian prior $N_A((\mathbf{z_i + m_{s_j}}), \mathbf{B})$ and a categorical likelihood $P(r_{ija}|$ $\mathbf{v_{ij}}$, $\mathbf{c})$. We have transformed the categorical likelihood to the conditional Gaussian form (recall Eqn. \ref{gauss_lik}). The posterior can be formulated as:
{\small
\vspace*{0.05in}
\begin{align}
\nonumber P(\mathbf{v_{ij}}) &\propto P(\mathbf{v_{ij}} | \mathbf{m_{s_j}}, \mathbf{z_i}, \mathbf{B}) * Lik(\mathbf{v_{ij}} | \boldsymbol{\omega},\mathbf{r_{ij}}, \mathbf{c})\\
&
\begin{aligned}
& \propto exp\{- \frac{1}{2} (\mathbf{v_{ij}} - (\mathbf{z_i} + \mathbf{m_{s_j}}))^T \mathbf{B}^{-1} (\mathbf{v_{ij}} - (\mathbf{z_i} + \mathbf{m_{s_j}}))\} \\ \nonumber
& \hspace{-0.1in} * \prod_{k=1}^{K-1} exp\{ -\frac{1}{2} (\frac{\boldsymbol{\kappa_{ij}^k}}{\boldsymbol{\omega_{ij}^k}} - (\mathbf{c_{k}} - \mathbf{v_{ij})})^T \boldsymbol{\Omega_{ij}^k}  (\frac{\boldsymbol{\kappa_{ij}^k}}{\boldsymbol{\omega_{ij}^k}} - (\mathbf{c_{k}} - \mathbf{v_{ij})}) \}
\end{aligned}
\end{align}
}%

Since both the prior and likelihood are now Gaussian, we have the following Gibbs sampler:
{\small
\begin{eqnarray*}
\bf{v_{ij}} &\sim& N_A(\boldsymbol{\mu_{ij\omega}} , \boldsymbol{\Sigma_{ij\omega}}) \\
\bf{\omega_{ija}} &\sim& PG(\bf{N_{ija}} ,v_{ija} - \bf{c}) 
\end{eqnarray*}
\vspace{-0.1in}
\noindent \text{where}
\vspace{-0.1in}
\begin{eqnarray*}
\boldsymbol{\mu_{ij\omega}} &=& \bf{B^{-1}}(\mathbf{z_i} + \mathbf{m_{s_j}}) + \sum_{k=1}^{K-1} \boldsymbol{\Omega_{ij}^k} (\bf{c_k} - \frac{\boldsymbol{\kappa_{ij}^k}}{\boldsymbol{\omega_{ij}^k}}) \\ 
\vspace{-0.2in}
\boldsymbol{\Sigma_{ij\omega}} &=& \bf{B^{-1}} + \sum_{k=1}^{K-1} \boldsymbol{\Omega_{ij}^k} 
\end{eqnarray*}
}%
\noindent\textbf{Sampling Cut-Points.}
Sigmoid function in the stick-breaking formulation allows us to sample cut-points while ensuring their relative order without additional constraints. Figure \ref{fig:cutPointSim} shows probability distributions for simulated cut-points. 

\begin{figure}[htbp]
\centering
	\includegraphics[width=0.65\linewidth]{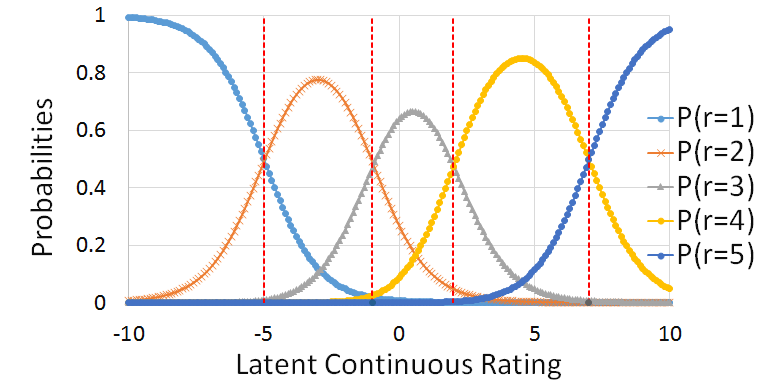}
	\vspace{-0.05in}
	\caption{\small Category probabilities for cut-points (-5,-1,2,7)}\label{fig:cutPointSim}
	\vspace*{-0.05in}
\end{figure}

The following lemma gives  the relationship between cut-points, latent continuous ratings, and the observed ratings.

\begin{lemma}
	
If  $v_{ija}$ $>$  $c_{k} - ln~(1-e^{-({c_{k+1}-c_{k})}}) $,  
then 
 $P(r_{ija}=k+1)$ $>$ $P(r_{ija}=k)$.
\label{cutpointlemma} 
\end{lemma}
\begin{proof}[Proof.]
\small
Let $\delta_{k} \geq -~ln~(1-e^{-(c_{k+1}-c_{k})})$.
By replacing $v_{ija}$ with ($c_{k}+\delta_{k}$)	in Eqn. 2, we have
	{\scriptsize
	\begin{align}
	\phantom{P(r_{ija}=(k))}
	&
	 \nonumber \mathllap{P(r_{ija}=k)} = \prod_{q<k} (1- f(c_{q} - c_{k} - \delta_{k}))(f(c_{k} - c_{k} - \delta_{k})) \\\nonumber
	&= \prod_{q<k} (1- f(c_{q} - c_{k} - \delta_{k}))(f(-\delta_{k}))\\ \nonumber	
	&	
		\mathllap{P(r_{ija}=k+1)}
		 = \prod_{q<k} (1- f(c_{q} - c_{k} - \delta_{k}))( 1 - f(-\delta_{k})) ( f(c_{k+1} - c_{k} -\delta_{k}))
	\end{align}
	}%
Taking the ratio, we have 
{\scriptsize 
\begin{align*}
\frac{P(r_{ija}=k+1)}{P(r_{ija}=k)} &= \frac{ ( 1 - f(-\delta_{k})) ( f(c_{k+1} - c_{k} -\delta_{k}))}{f(-\delta_{k})}\\
&= (\frac{e^{\delta_{k}}}{1+ e^{\delta_{k}}} * \frac{1}{1+e^{c_{k}+\delta_{k} - c_{k+1}}}) / (\frac{1}{1+e^{\delta_{k}}})\\
&=  \frac{e^{\delta_{k}}} { 1+e^{c_{k}-c_{k+1}+\delta_{k}}}
\end{align*}
 }%
 
Since $\delta_{k} \geq -ln(1- e^{-(c_{k+1}-c_{k})})$, we see that $\frac{e^{\delta_{k}}} { 1+e^{c_{k}-c_{k+1}+\delta_{k}}} > 1$. Hence,  $P(r_{ija}=k+1) > P(r_{ija}=k)$. 
\end{proof}

We have shown that $P(r_{ija}=k+1) > P(r_{ija}=k)$ when $v_{ija} \geq (c_{k} -ln(1- e^{-(c_{k+1}-c_{k})})$. Similarly, $P(r_{ija}=k) > P(r_{ija}=k-1)$ when $v_{ija} \geq (c_{k-1} -ln(1- e^{-(c_{k}-c_{k-1})}))$. 
This implies that, when $v_{ija}$ is within the range  $(c_{k-1}-ln(1- e^{-(c_{k}-c_{k-1})}) , c_k -ln(1- e^{-(c_{k+1}-c_{k})})]$, then $P(r_{ija}=k)$  has the maximum probability over all other categories. In other words,  for $v_{ija}$ in the stated range, we have $\argmax_{k'} P(r_{ija} | v_{ija},k') = k$.

Hence, given the sampled values of $v_{ija}$ we can constrain the possible set of values for the cut-points. We sample cut-point $c_k$ from a uniform distribution within the  range:
{\scriptsize
\begin{align*}
c_k & \sim U[max\{v_{ija}|  \argmax_{k'} P(r_{ija} | v_{ija},k') = k\} -ln(1- e^{-(c_{k}-c_{k-1})}), \\ 
& \qquad min\{v_{ija} | \argmax_{k'} P(r_{ija} | v_{ija},k') = k+1\} - ln(1- e^{-(c_{k}-c_{k-1})})] 
\end{align*}
}%

\section{Experiments}
For evaluation we use hotel ratings from TripAdvisor \cite{wang2011latent} and restaurant ratings from Opentable.com.
We crawled OpenTable.com for all the restaurant ratings in New York Tri-State area.
Table \ref{dataset} shows the details of the datasets. 
\begin{table}[thbp]
\scriptsize
\centering
\begin{tabular}{|c|c|c|c|l|}
\hline
Dataset    & \# Items & \# Users & \# Ratings & Aspects rated\\ \hline
TripAdvisor& 12,773      & 781,403    & 1,621,956 &  Service, Value, Room, Location \\ \hline
OpenTable       & 2805       & 1997    & 73,469 &  Ambience, Food, Service, Value\\ \hline
\end{tabular}
\vspace{-0.05in}
\caption{\small Statistics of experimental datasets.}
\label{dataset}
\end{table}
\vspace*{-0.1in}

\subsection{Rating Prediction}

One application of Ordinal Aspect Bias model is predicting observed aspect ratings.
We perform five-fold cross validation on user-item pairs, and take expected value of an aspect rating as the predicted rating.
Note that all the aspect  ratings for the same user-item pair will be in the same training or test set.
By default, the number of user groups are set to 10.
For comparison, we also implemented the following models:

\begin{itemize}[leftmargin=*]
\itemsep0em
\item \textbf{Continuous Aspect Bias  model} is the continuous variant of our model where  observed ratings are assumed to be continuous. Observed ratings are drawn from a (conjugate) multivariate Gaussian distribution, with mean as the true rating of the item offset with the bias of the user's group.
\item \textbf{Ordinal} and \textbf{Continuous No Bias model} assume users are not biased. The observed ratings for an item are drawn from only the true rating of the item. 
\item \textbf{Ordinal} and \textbf{Continuous Global Bias model} assume all users have the same bias. All ratings for an item are drawn from the true rating of the item offset with a global bias. 
\end{itemize}
\vspace{-0.05in}

\begin{table}[h]
	\scriptsize
\centering
\begin{tabular}{|l|l|l|l|l|}
\hline
\multirow{2}{*}{Model} & \multicolumn{2}{l|}{TripAdvisor Data} & \multicolumn{2}{l|}{OpenTable Data} \\  \cline{2-5}
 & log LL & RMSE & log LL & RMSE \\ \hline
Ordinal Aspect Bias & \bf{-557.08} & \bf{1.00} & \textbf{-493.79} &  \bf{1.03}\\ \hline
Continuous Aspect Bias & -1050.32 & 3.13 & -560.14 & 2.21\\ \hline
Ordinal No Bias & -689.76 & 1.47 &  -546.25 & 1.95 \\ \hline
Continuous No Bias & -1904.64 & 3.52 &  -651.16 & 2.39 \\ \hline
Ordinal Global Bias	&	-2438.52 & 2.85 & -570.28 & 2.37\\	\hline
Continuous  Global Bias	&	-2632.95 & 3.91 & -595.62 & 2.41\\	\hline
\end{tabular}
\caption{\small Test set log likelihood (the higher, the better)  and RMSE (the lower, the better). All comparisons are statistically significant (paired \textit{t-test} with $p < 0.0001$).}
\label{testLL}
\end{table} 
 
\begin{table*}[t]
\centering
\resizebox{\textwidth}{!}{%
\begin{tabular}{|c|cc|cc|cc|cc|cc|cc|cc|cc|}
\hline
\multirow{2}{*}{Model} & \multicolumn{8}{c|}{TripAdvisor Data} & \multicolumn{8}{c|}{OpenTable Data} \\ 
 & \multicolumn{2}{c}{Service} & \multicolumn{2}{c}{Value} & \multicolumn{2}{c}{Room} & \multicolumn{2}{c|}{Location} & \multicolumn{2}{c}{Ambience} & \multicolumn{2}{c}{Food} & \multicolumn{2}{c}{Service} & \multicolumn{2}{c|}{Value} \\ \hline
 & RMSE & FCP & RMSE & FCP & RMSE & FCP & RMSE & FCP & RMSE & FCP & RMSE & FCP & RMSE & FCP & RMSE & FCP \\ \cline{2-17}
PMF & 2.006 & 0.501 & 1.933 & 0.526 & 1.836 & 0.592 & 2.127 & 0.603 & 2.584 & 0.524 & 2.232 & 0.530 & 2.388 & 0.511 & 2.151 & 0.521 \\
BPMF & 1.414 & 0.586 & 1.373 & 0.571 & 1.314 & 0.614 & 1.209 & 0.651 & 1.154 & 0.490 & 0.992 & 0.532 & 1.426 & 0.498 & 1.302 & 0.519 \\
URP & 1.179 & 0.489 & 1.156 & 0.515 & 1.194 & 0.513 & 1.001 & 0.492 & 0.952 & 0.557 & 0.818 & 0.551 & 1.144 & 0.522 & 1.120 & 0.514 \\
SVD++ & \textbf{1.064} & 0.578 & 1.079 & 0.562 & 1.093 & 0.639 & 0.894 & 0.665 & \textbf{0.944*} & 0.525 & 0.831 & 0.544 & \textbf{1.088} & 0.544 & 1.131 & 0.517 \\
BHFree & 1.143 & 0.553 & 1.199 & 0.582 & 1.124 & 0.624 & 1.007 & 0.671 & 0.956 & 0.483 & 0.812 & 0.499 & 1.151 & 0.512 & 1.096 & 0.495 \\
LARA & 1.193 & 0.576 & 1.221 & 0.531 & 1.087 & 0.558 & 1.170 & 0.672 & 1.150 & 0.538 & 2.242 & 0.514 & 2.444 & 0.549 & 1.089 & 0.526 \\
OrdRec + SVD++ & 1.348 & 0.619 & 1.344 & 0.613 & 1.359 & 0.654 & 1.173 & 0.702 & 1.337 & 0.672 & 1.121 & 0.613 & 1.533 & 0.618 & 1.521 & 0.623 \\
AspectBias & 1.067 & \textbf{0.646*} & \textbf{1.063*} & \textbf{0.645*} & \textbf{1.045} & \textbf{0.678*} & \textbf{0.854*} & \textbf{0.717} & 0.953 & \textbf{0.854*} & \textbf{0.787*} & \textbf{0.850*} & 1.134 & \textbf{0.842*} & \textbf{1.043*} & \textbf{0.864*}
 \\ \hline
\end{tabular}
}%
\vspace{-0.05in}
\caption{\small Rating Prediction RMSE (the lower, the better) and FCP (the higher, the better) results. "*" denotes statistical significance with the runner up for $p < 0.005$}
\label{ratingPred}
\vspace{-0.1in}
\end{table*}

 Table \ref{testLL} shows mean log likelihood and RMSE (root mean square error) on test data. For both datasets Ordinal Aspect Bias model performs the best, demonstrating the need to consider both user bias and the proper ordinal nature of ratings.

Next, we compare the performance of our model with state-of-the-art rating prediction models, namely, PMF \cite{salakhutdinov2011probabilistic}, BPMF \cite{salakhutdinov2008bayesian}, URP \cite{marlin2003modeling,barbieri2011regularized}, SVD++ \cite{koren2008factorization} and BHFree \cite{pierre2012balancing}. For each of these, we used the best parameter settings published on LibRec.net
website. We also compare with OrdRec \cite{koren2013collaborative} which can wrap existing collaborating filtering  methods such as SVD++ \cite{koren2008factorization} to tackle ordinal rating. Since these models cannot predict multiple aspect ratings for a user-item pair, we train them separately for each aspect.  We further compare with LARA \cite{wang2011latent} which models latent aspect ratings using review texts. 
Since RMSE cannot capture personalization or ordinal rating values, we also use FCP to measure the fraction of correctly ranked pair of items for each user \cite{koren2013collaborative}. Table \ref{ratingPred} shows the results for both datasets. We see that the proposed model  outperforms state-of-the art  methods in most cases.

\begin{table}[htbp]
\scriptsize
	\centering
	\begin{tabular}{|c|c|c|}
		\hline
		Method & TripAdvisor Data & OpenTable Data \\ \hline
		PMF&0.016& 0.142\\ \hline
		BPMF&0.219&0.133\\ \hline
		URP&0.238&0.177\\ \hline
		SVD++ & 0.364 & 0.201\\ \hline
		BHFree&0.359&0.205\\ \hline
		LARA&0.289&0.152\\ \hline
		OrdRec + SVD++ &0.148&0.262\\ \hline
		OrdinalAspectBias & \textbf{0.404} & \textbf{0.298} \\ \hline
	\end{tabular}
	\vspace*{-0.05in}
	\caption{\small Pearsons Correlation of aspect ranking}
	\vspace*{-0.05in}
	\label{aspectRanking} 
\end{table}

The relative ranking of aspects for a user-item pair is also important to understand which aspects of an item the user liked better. For different methods table \ref{aspectRanking}  shows the  Pearson correlation coefficient of aspect ranking for a user-item pair, compared to its ground truth ranking. Clearly, Ordinal Aspect Bias model outperforms all other methods for the task of relative ranking of aspects.
This validates that our model is able to learn aspect rating behavior of users accurately.

\subsection{Evaluation of User Groups}

A significant advantage of our model is that it can infer latent user groups depending on their rating behaviors across multiple items. In this set of experiments, we show that
if users are assigned to the same group, then their ratings on  the same items for the same aspects are similar.

 We look at the standard deviation of the set of users belonging to the same group who have rated the same entity
 \cite{wang2011latent}. For each aspect of each item, we compare the standard deviation of the ratings of each user group with that of a control group comprising of all the users who have rated the item.

\begin{figure}[htbp]
    \centering	
	\vspace*{-0.1in}
	\subfloat[TripAdvisor]{\includegraphics[height = 1.1in,width=3.4cm]{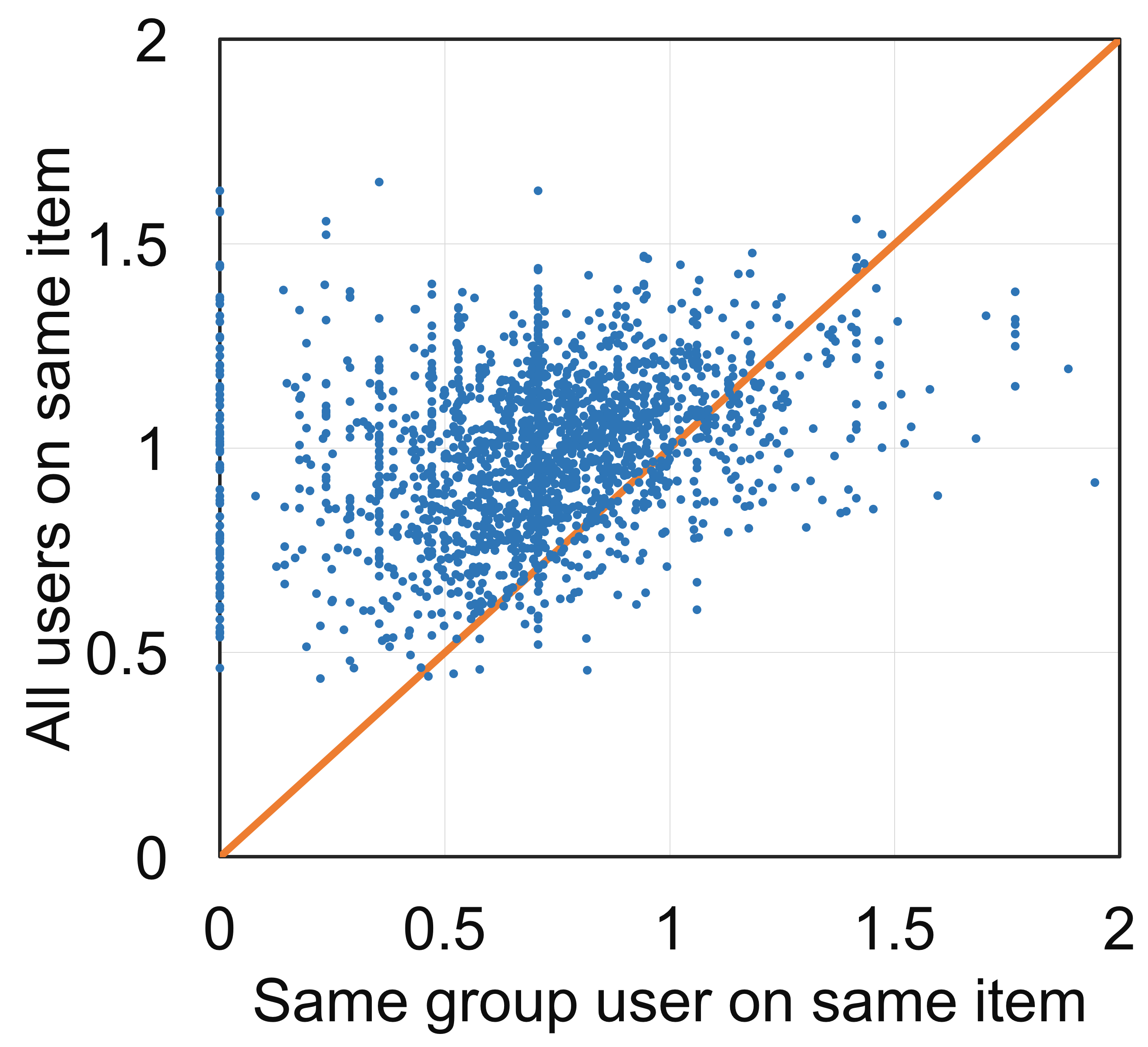}\label{SdTAall}}
    \hspace{0.5cm} \subfloat[OpenTable]{\includegraphics[height = 1.1in,width=3.4cm]{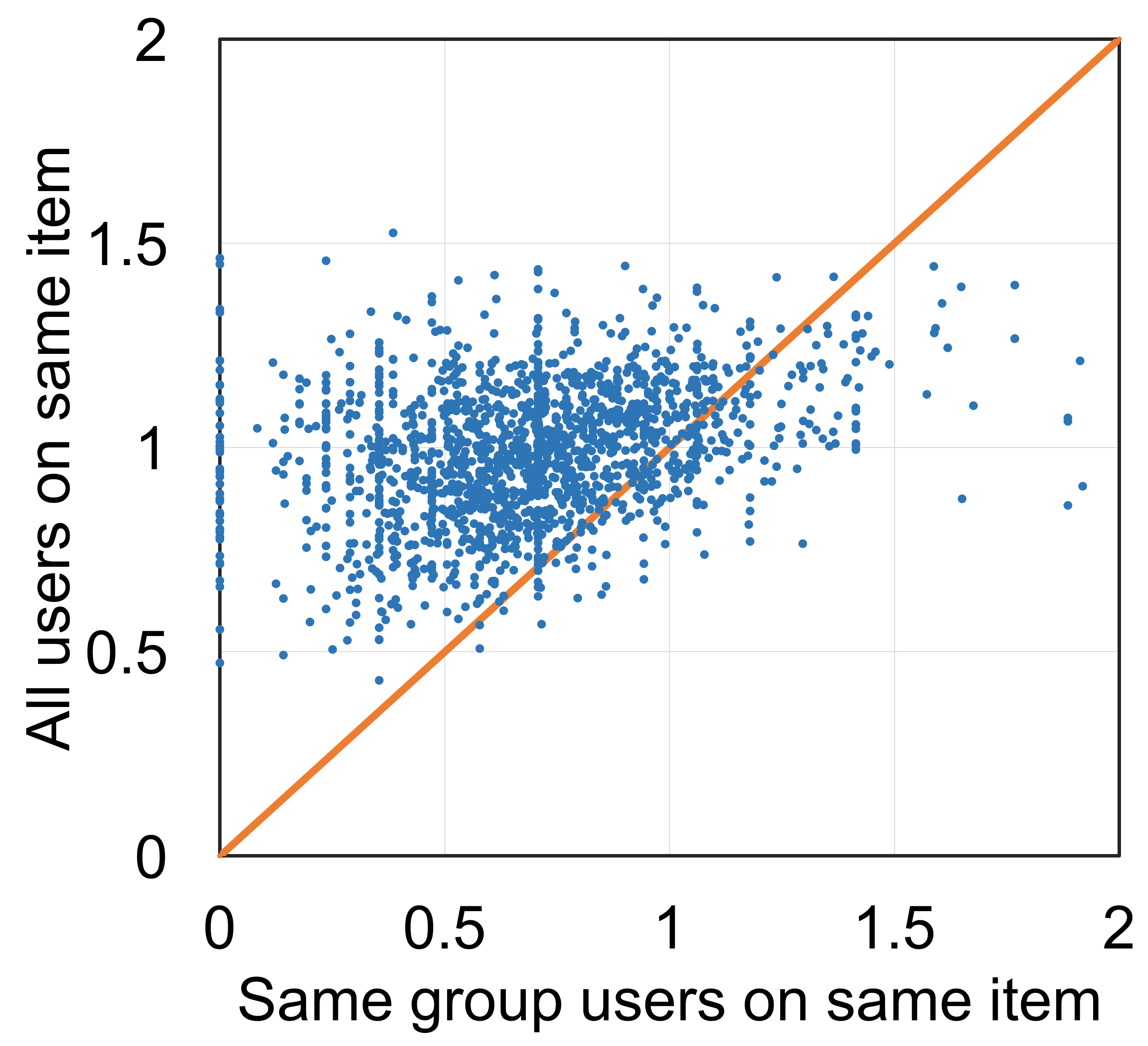}\label{SdOTall}}
	\vspace{-0.05in}
	\caption{\small Scatter plot of standard deviations of aspect ratings.}\label{clusterSD}
	\vspace*{0.1in}
\end{figure}

Figure \ref{clusterSD} shows the scatter plots of the standard deviations for both datasets. We observe that most of the points lie above the line $y=x$, indicating that users who belong to the same group have smaller standard deviation compared to the control group. This implies that the latent user groups obtained by the proposed model 
can effectively cluster users who give similar aspect ratings to the same item.

\begin{figure}[ht]
	\centering
	\subfloat[\small TripAdvisor]{ \includegraphics[height=0.22\linewidth ,width=0.75\linewidth ]{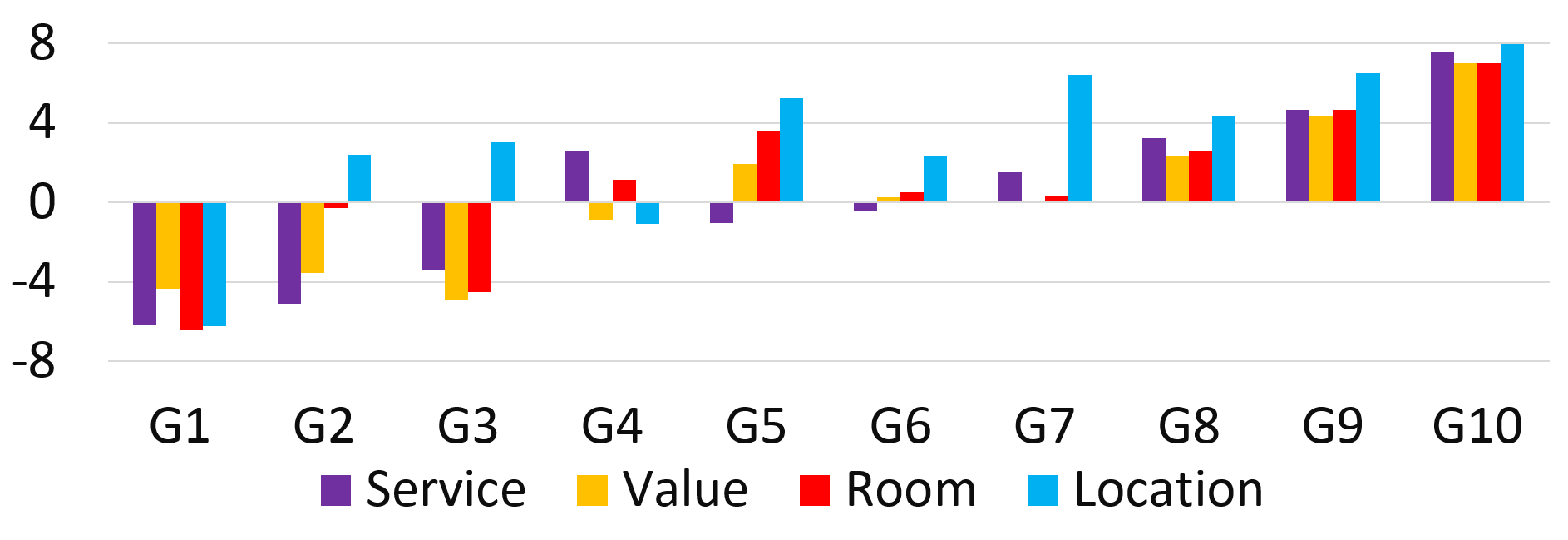}\label{fig:groupBiasTA}}\\
	\vspace{-0.15in}
	\subfloat[\small OpenTable]{  \includegraphics[height=0.22\linewidth ,width=0.75\linewidth]{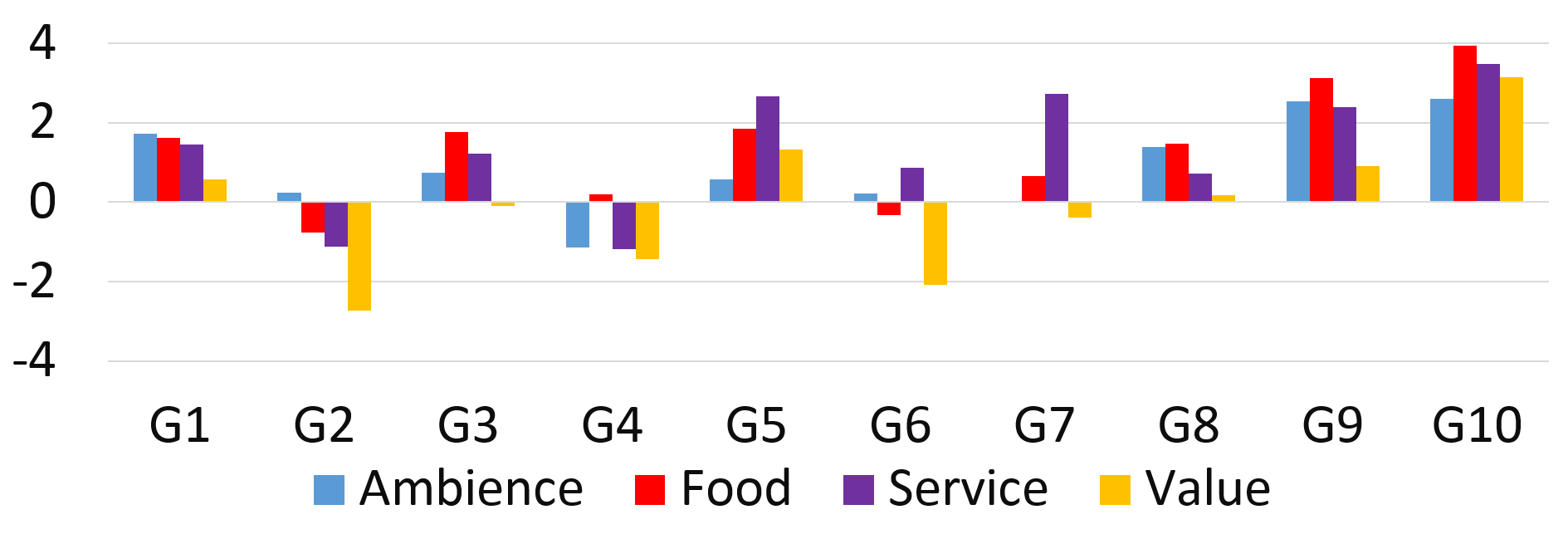}\label{fig:groupBiasOT}}
	\vspace*{-0.05in}
	\caption{\small Mean bias value of user groups.}\label{fig:groupBias}
	\vspace*{0.1in}
\end{figure}

Figure \ref{fig:groupBias} shows the mean ratings of 10 user groups after scaling the ratings to the range [-10, 10].
For the TripAdvisor dataset, we observe that the first user group seems to be quite critical whereas the last three user groups are positive. 
We also see correlation of aspect biases for different groups. 
For example, group 5 and 7 seem
to have similar biases for \textit{Room} and \textit{Value} whereas group 4 is demanding about \textit{Value} and \textit{Location}. Considering all
the ratings of the users belonging to group 5 and 7, we see that their
ratings for \textit{Value} are indeed most correlated with their ratings for
\textit{Room} than other aspects. On the other hand for group 4 their ratings
for \textit{Value} are highly correlated with their ratings for \textit{Location}. 
This suggests that for good
\textit{Value} for money,  some users prefer good \textit{Location} while some prioritize better \textit{Room} quality
and by modeling the covariance structure among aspects we are able to
uncover such dependencies.
For the OpenTable dataset, we see that users in group 4 who are particular about  \textit{Ambience} are also demanding about \textit{Service} and  \textit{Value}.

\subsection{Intrinsic Quality of Items}
Often one forms a judgment about the quality of an item by the average  rating it has received. However, if an item has
 received only a few ratings,  it is difficult to form an accurate opinion concerning its quality.
In this set of experiments, we show that the intrinsic quality, learned by the proposed model, is correlated with   
users' perception of the item's true quality, even for items with few ratings. 

We focus on items with less than 30 ratings and whose intrinsic quality and average rating for an aspect differ by at least 0.5. Since an item's true quality is unknown, we estimate it by the relative difference in the observed ratings of the same user on a pair of items. This is because if the qualities of two items are similar, a user will rate them similarly. 

For each pair of items rated by the same user on the same aspect, let their difference in observed ratings be $\Delta obs$, difference between their average ratings be $\Delta avg$ and difference between the learned intrinsic ratings be $\Delta int$. 
Figure \ref{correlIntAvg} shows the correlation between $\Delta obs$ and $\Delta int$, as well as the correlation between $\Delta obs$ and $\Delta avg$ aggregated over all aspects. We observe that for both datasets, as $\Delta int$ increases, $\Delta obs$ also increases. However, $\Delta avg$ remains almost constant. This indicates that $\Delta obs$ is closely correlated with $\Delta int$, whereas $\Delta avg$ appears to be independent of $\Delta obs$. This confirms that the learned intrinsic rating is better able to  reflect users' perception of the true quality of an item compared to using average ratings of the items. 

\begin{figure}[hthb]
	\centering
	\subfloat[OpenTable]{\includegraphics[width=0.4\linewidth]{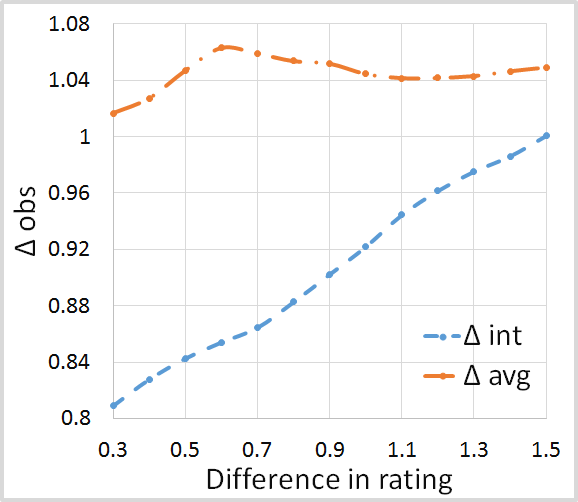}\label{TrTAInt}}
	\hspace{0.5cm} \subfloat[TripAdvisor]{\includegraphics[width=0.4\linewidth]{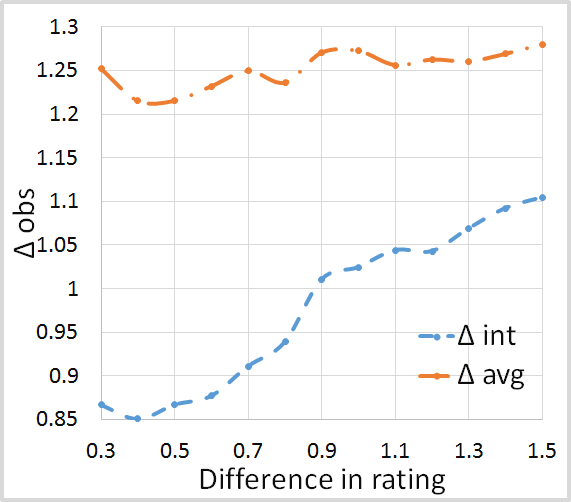}\label{TrTAExt}}
	\vspace{-0.05in}
	\caption{\small Correlation with $\Delta obs$}\label{correlIntAvg}
	\vspace{-0.15in}
\end{figure}

\subsection{Case Study}

Finally, we present the reviews of a user from OpenTable to demonstrate that the aspect bias learned by our model correlates with their review texts (see Figure \ref{criticalCaseStudy}). The user is from   group G2 in Figure~\ref{fig:groupBias}
that is particularly critical about \textit{Value}.

From the reviews of this user, as well as the reviews of randomly selected users from other groups for the same item, we see that the user from group 2 is indeed critical. We further confirm this observation by  manually going through 100 randomly sampled reviews and tabulate the sentiment distribution of each item.  We observe that the user is consistently critical even though the majority opinion is positive. 
This strengthens the fact that the group bias captured by our model is accurate and can help us better interpret a users' rating.

\begin{figure}[thbp]
	\centering
	\includegraphics[width=0.95\linewidth]{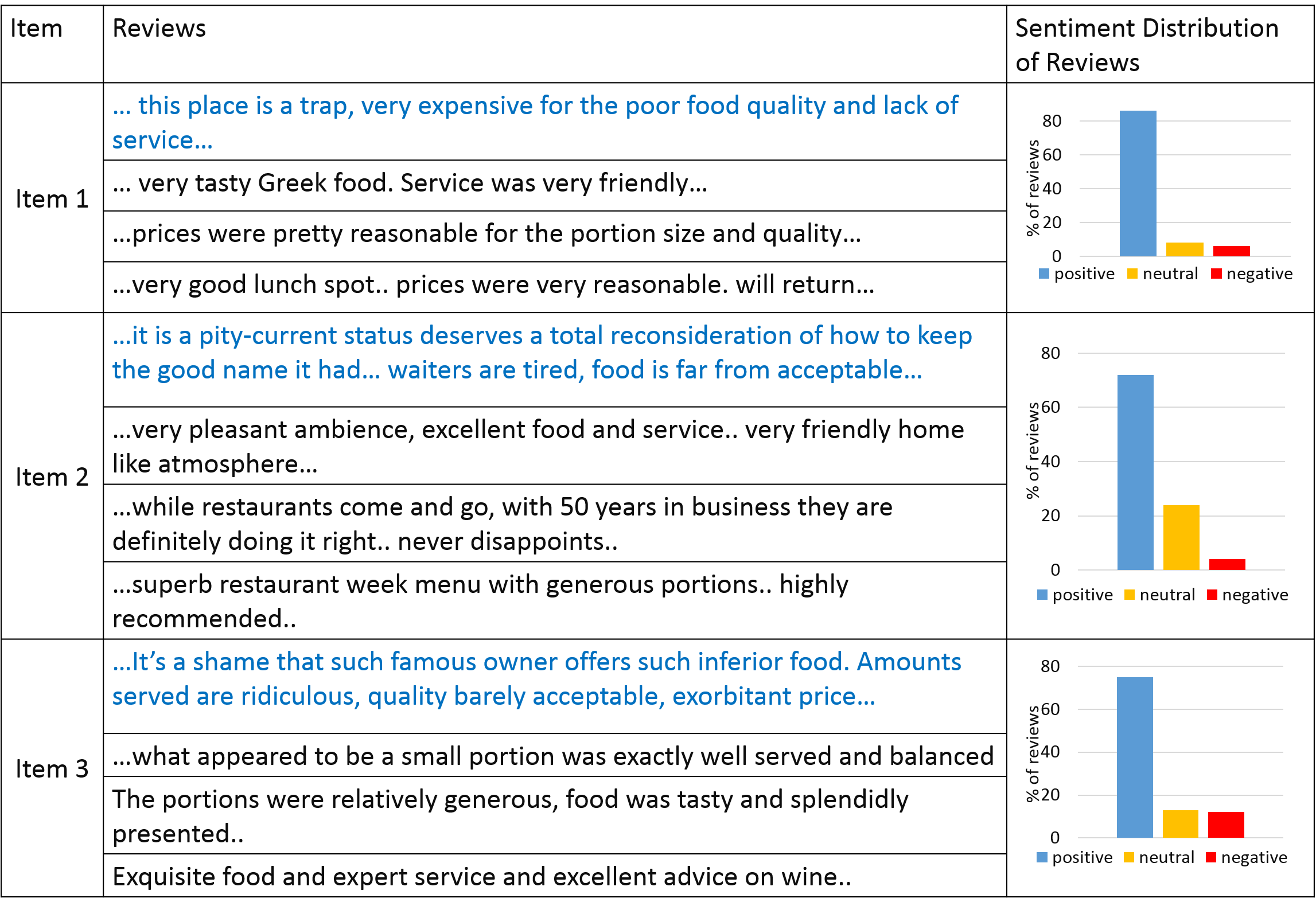}
	\vspace{-0.05in}	
	\caption{\small Reviews of user belonging to "critical"  group contrasted with other reviews on the same items}
	\label{criticalCaseStudy}
	\vspace*{0.1in}
\end{figure}

\section{Related Work}

Existing works on aspect rating prediction use reviews to analyze latent aspect ratings \cite{hao2014sentiment,wang2011latent} and ignore the \emph{explicit} aspect ratings provided by users. While the widely used CF approaches for rating prediction view ratings as \emph{continuous} values and do not encode aspect dependencies
\cite{lee2001algorithms,koren2008factorization,hofmann2003collaborative,marlin2003modeling,pierre2012balancing,salakhutdinov2008bayesian,salakhutdinov2011probabilistic,pennock2000collaborative,wang2009latent}.

There have been very few attempts to address the ordinal nature of ratings. The authors in  \cite{stern2009matchbox} develop a model combining CF and content-based filtering using regression to handle ordinal ratings as a special case. The work in \cite{koren2013collaborative}  proposes a \textit{wrapper} around a CF method for ordinal data. 
Both of these works use a logit model for ordinal regression. 

In contrast,  most statistical approaches handle ordinal data using an ordinal probit model \cite{albert1993bayesian,rossi2001overcoming,muthukumarana2014bayesian}. Although they allow a Bayesian inference but it necessitates using truncated Gaussian distributions and forced ordering of cut-off points. This leads to complicated and even sub-optimal inference. 

The authors of \cite{virtanen2015ordinal} used  stick-breaking formulation to parameterize the underlying continuous rating. However, since the non-conjugacy made an MCMC sampling non-trivial, they performed an approximate variational Bayesian inference. For correlated topic models \cite{chen2013scalable}, P\'{o}lya-Gamma auxiliary variable augmentation is used with logistic-normal transformation, whereas the work in \cite{khan2012stick} used stick-breaking likelihood for categorical data. However, none of these works use stick-breaking likelihood with a  P\'{o}lya-Gamma variable augmentation to exploit conjugacy to facilitate Gibbs sampling.

\section{Conclusion}\label{conclusion}

We have presented a novel approach to understand users' aspect bias, while capturing aspect dependencies as well as the proper ordinal nature of user responses. Our construction of the stick-breaking likelihood coupled with P\'{o}lya-Gamma auxiliary variable augmentation has resulted in an elegant Bayesian inference of the model.

 Empirical evaluation  on two real world datasets
demonstrates that through proper statistical modeling of data we are able to  capture users' rating behavior and outperform state-of-the-art approaches. Furthermore, our model is effective in user modeling, analyzing users' aspect preferences and provides a better product quality estimation even when the product has received few ratings. Most importantly, the construction of the model described here is generic and presents new possibilities for modeling such data in a wide-range of domains. Our work is orthogonal to works involving texts and social graph of rating domains and it will be interesting to know the connection between bias groups and social groups.

\newpage
\bibliographystyle{abbrv}

\bibliography{references}

\begin{thebibliography}{10}

\bibitem{albert1993bayesian}
J.~H. Albert and S.~Chib.
\newblock Bayesian analysis of binary and polychotomous response data.
\newblock {\em Journal of the American statistical Association}, 1993.

\bibitem{barbieri2011regularized}
N.~Barbieri.
\newblock Regularized gibbs sampling for user profiling with soft constraints.
\newblock In {\em Advances in Social Networks Analysis and Mining (ASONAM)},
  2011.

\bibitem{chen2013scalable}
J.~Chen, J.~Zhu, Z.~Wang, X.~Zheng, and B.~Zhang.
\newblock Scalable inference for logistic-normal topic models.
\newblock In {\em Advances in Neural Information Processing Systems}, 2013.

\bibitem{hofmann2003collaborative}
T.~Hofmann.
\newblock Collaborative filtering via gaussian probabilistic latent semantic
  analysis.
\newblock In {\em Proceedings of the 26th annual international ACM SIGIR
  conference on Research and development in informaion retrieval}, 2003.

\bibitem{khan2012stick}
M.~E. Khan, S.~Mohamed, B.~M. Marlin, and K.~P. Murphy.
\newblock A stick-breaking likelihood for categorical data analysis with latent
  gaussian models.
\newblock In {\em AISTATS}, 2012.

\bibitem{koren2008factorization}
Y.~Koren.
\newblock Factorization meets the neighborhood: a multifaceted collaborative
  filtering model.
\newblock In {\em ACM SIGKDD}, 2008.

\bibitem{koren2013collaborative}
Y.~Koren and J.~Sill.
\newblock Collaborative filtering on ordinal user feedback.
\newblock In {\em IJCAI}, 2013.

\bibitem{lee2001algorithms}
D.~D. Lee and H.~S. Seung.
\newblock Algorithms for non-negative matrix factorization.
\newblock In {\em Advances in neural information processing systems}, 2001.

\bibitem{marlin2003modeling}
B.~M. Marlin.
\newblock Modeling user rating profiles for collaborative filtering.
\newblock In {\em Advances in neural information processing systems}, 2003.

\bibitem{muthukumarana2014bayesian}
S.~Muthukumarana and T.~B. Swartz.
\newblock Bayesian analysis of ordinal survey data using the dirichlet process
  to account for respondent personality traits.
\newblock {\em Communications in Statistics-Simulation and Computation}, 2014.

\bibitem{pennock2000collaborative}
D.~M. Pennock, E.~Horvitz, S.~Lawrence, and C.~L. Giles.
\newblock Collaborative filtering by personality diagnosis: A hybrid memory-and
  model-based approach.
\newblock In {\em Uncertainty in artificial intelligence}, 2000.

\bibitem{pierre2012balancing}
C.~Pierre.
\newblock Balancing prediction and recommendation accuracy: hierarchical latent
  factors for preference data.
\newblock 2012.

\bibitem{polson2013bayesian}
N.~G. Polson, J.~G. Scott, and J.~Windle.
\newblock Bayesian inference for logistic models using p{\'o}lya--gamma latent
  variables.
\newblock {\em Journal of the American statistical Association}, 2013.

\bibitem{rossi2001overcoming}
P.~E. Rossi, Z.~Gilula, and G.~M. Allenby.
\newblock Overcoming scale usage heterogeneity: A bayesian hierarchical
  approach.
\newblock {\em Journal of the American Statistical Association}, 2001.

\bibitem{salakhutdinov2011probabilistic}
R.~Salakhutdinov and A.~Mnih.
\newblock Probabilistic matrix factorization.
\newblock In {\em NIPS}, 2007.

\bibitem{salakhutdinov2008bayesian}
R.~Salakhutdinov and A.~Mnih.
\newblock Bayesian probabilistic matrix factorization using markov chain monte
  carlo.
\newblock In {\em International conference on Machine learning}, 2008.

\bibitem{stern2009matchbox}
D.~H. Stern, R.~Herbrich, and T.~Graepel.
\newblock Matchbox: large scale online bayesian recommendations.
\newblock In {\em Proceedings of the 18th international conference on World
  wide web}, 2009.

\bibitem{virtanen2015ordinal}
S.~Virtanen and M.~Girolami.
\newblock Ordinal mixed membership models.
\newblock In {\em International Conference on Machine Learning}, 2015.

\bibitem{hao2014sentiment}
H.~Wang and M.~Ester.
\newblock A sentiment-aligned topic model for product aspect rating prediction.
\newblock In {\em EMNLP}, 2014.

\bibitem{wang2011latent}
H.~Wang, Y.~Lu, and C.~Zhai.
\newblock Latent aspect rating analysis without aspect keyword supervision.
\newblock In {\em Proceedings of the 17th ACM SIGKDD}, 2011.

\bibitem{wang2009latent}
P.~Wang, C.~Domeniconi, and K.~B. Laskey.
\newblock Latent dirichlet bayesian co-clustering.
\newblock In {\em Joint European Conference on Machine Learning and Knowledge
  Discovery in Databases}, 2009.

\end{thebibliography}

\end{document}